\documentclass{article}

\usepackage[preprint]{neurips_2025}
\usepackage{amsmath, amssymb}
\usepackage{amsfonts}
\usepackage{pgfplots}
\usepackage{algorithm}
\usepackage{algpseudocode}
\usepackage{mathtools}

\pgfplotsset{compat=1.18}

\usepackage{tikz}
\usepackage[utf8]{inputenc} 
\usepackage[T1]{fontenc}    
\usepackage{hyperref}       
\usepackage{url}            
\usepackage{booktabs}       
\usepackage{amsfonts}       
\usepackage{nicefrac}       
\usepackage{microtype}      
\usepackage{xcolor}         
\usepackage{subfigure}

\usepackage{amsthm}
\newtheorem{theorem}{Theorem}[section]
\newtheorem{lemma}[theorem]{Lemma}
\newtheorem{proposition}[theorem]{Proposition}

\newtheorem{definition}[theorem]{Definition}

\title{Learning Pareto-Optimal Rewards from Noisy Preferences: A Framework for Multi-Objective Inverse Reinforcement Learning}

\author{%
  Kalyan Cherukuri\\
  Department of Computer Science\\
  Illinois Mathematics and Science Academy\\
  Aurora, IL 60502 \\
  \texttt{kcherukuri@imsa.edu}
  \and
  \textbf{Aarav Lala}\\
  Department of Computer Science\\
  Illinois Mathematics and Science Academy\\
  Aurora, IL 60502 \\
  \texttt{alala1@imsa.edu}
}

\begin{document}

\maketitle

\begin{abstract}
As generative agents become increasingly capable, alignment of their behavior with complex human values remains a fundamental challenge. Existing approaches often simplify human intent through reduction to a scalar reward, overlooking the multi-faceted nature of human feedback. In this work, we introduce a theoretical framework for preference-based Multi-Objective Inverse Reinforcement Learning (MO-IRL), where human preferences are modeled as latent vector-valued reward functions. We formalize the problem of recovering a Pareto-optimal reward representation from noisy preference queries and establish conditions for identifying the underlying multi-objective structure. We derive tight sample complexity bounds for recovering $\epsilon$-approximations of the Pareto front and introduce a regret formulation to quantify suboptimality in this multi-objective setting. Furthermore, we propose a provably convergent algorithm for policy optimization using preference-inferred reward cones. Our results bridge the gap between practical alignment techniques and theoretical guarantees, providing a principled foundation for learning aligned behaviors in a high-dimension and value-pluralistic environment.
\end{abstract}

\section{Introduction}

Inverse Reinforcement Learning (IRL) has emerged as a principled approach to infer reward functions from expert behavior or human preferences. While existing IRL methods typically assume scalar rewards, such simplifications can obscure essential trade-offs in human value systems. Multi-Objective Reinforcement Learning (MORL), by contrast, models reward as a vector of competing objectives, enabling richer and more nuanced representations of agent incentives. However, integrating multi-objective reasoning with preference-based learning remains underexplored, particularly from a theoretical perspective.

In this work, we introduce a novel framework titled Learning Pareto-Optimal Rewards from Noisy Preferences, which addresses this gap by combining multi-objective inverse reinforcement learning (MO-IRL) with preference-based feedback. Our goal is to infer latent, vector-valued reward functions from noisy comparisons over policies or trajectories. We argue that preference feedback—already widely used in fine-tuning generative models—naturally encodes information about trade-offs between objectives, making it well-suited to the MO-IRL setting. While the focus of this work is on theoretical development, we highlight that this framework has potential applications in aligning complex agents, such as generative models, where learning from nuanced human feedback is critical.

Our framework formalizes the problem of learning a Pareto-optimal reward representation from noisy preference queries and provides rigorous guarantees on identifiability and sample complexity. We derive tight sample complexity bounds for recovering $\epsilon$-approximations of the Pareto front, revealing how many comparisons are needed to extract meaningful multi-objective signals. Additionally, we introduce a novel regret formulation for evaluating policy suboptimality in the multi-objective setting, and propose an efficient algorithm for policy optimization under inferred reward cones.

\section*{Contributions}

\begin{itemize}
\item We formulate the problem of MO-IRL from human preferences, characterizing the conditions under which the underlying multi-objective reward structure is identifiable from noisy comparisons.
\item We derive tight sample complexity bounds for recovering $\epsilon$-approximations of the Pareto front, revealing how many comparisons are needed to learn meaningful multi-objective signals.
\item We introduce a novel regret formulation for evaluating policy suboptimality in the multi-objective setting, and propose an efficient algorithm for policy optimization under inferred reward cones.
\end{itemize}

By grounding alignment in the mathematical theory of multi-objective optimization and preference learning, our work bridges the gap between practical alignment methods and theoretical guarantees. Our theory applies directly to settings like preference alignment in LLMs (e.g., helpfulness vs. harmlessness), policy shaping in assistive robotics, and multi-goal tuning of generative models.

\section{Related Work}
\subsection{Inverse Reinforcement Learning and Preference Learning}

Inverse Reinforcement Learning (IRL) focuses on recovering reward functions from observed behavior, enabling agents to infer objectives without direct specification \cite{ng2000algorithms, abbeel2004apprenticeship}. Probabilistic approaches such as maximum entropy IRL handle reward ambiguity by maximizing entropy over possible explanations \cite{ziebart2008maximum}. Recent progress leverages deep learning to infer rewards from human preferences, improving robustness and generalization in complex environments \cite{christiano2017deep, brown2019extrapolating}. Preference elicitation has also benefited from active learning strategies, where queries are selected adaptively to efficiently identify reward functions or rankings \cite{sadigh2017active, jamieson2011active, yue2012beat}. Inverse reward design further explores how reward functions can be specified indirectly to avoid unintended behaviors \cite{hadfield2017inverse}.

\subsection{Multi-Objective Reinforcement Learning}

Multi-objective reinforcement learning (MO-RL) addresses problems where multiple conflicting criteria must be optimized simultaneously \cite{roijers2013survey, van2014multi}. Traditional methods often scalarize multiple objectives into a single reward, which may limit exploration of the full Pareto front. Recent advances introduce preference-driven MO-RL algorithms that incorporate explicit preference models to more effectively navigate trade-offs between objectives \cite{basaklar2023pdmorl, 10865310}. Prediction-guided approaches enable efficient multi-objective policy optimization in continuous control tasks \cite{pmlr-v119-xu20h}. These developments are supported by efficient reinforcement learning techniques tailored for structured or constrained settings \cite{chen2021efficient}.

\subsection{AI Alignment and Value Learning}

The alignment of AI systems with human values emphasizes the importance of learning and adhering to preferences in a safe and interpretable manner \cite{russell2019human, soares2015corrigibility}. Preference-based learning paradigms that directly incorporate human feedback are promising pathways toward ensuring value-aligned behavior \cite{christiano2017deep}. Theoretical foundations in convex analysis and optimization provide essential tools for developing algorithms with provable guarantees \cite{hiriart1993convex}. Together, these approaches contribute to building AI systems that can reliably respect complex, multi-faceted human values while maintaining robustness in uncertain environments.

\section{Preliminaries}
Here, we will define and introduce fundamental concepts to understand the problem setting and key contributions. 

\subsection{Inverse Reinforcement Learning (IRL)}

IRL aims to recover the reward function of an agent by observing its behavior. Unlike standard reinforcement learning (RL), where the reward function is given, IRL assumes the reward function is unknown, and the goal is to infer it from expert demonstrations. The core idea is that an agent’s observed behavior (e.g., trajectory or policy) is assumed to be optimal with respect to some unknown reward function. Formally, given a trajectory \( \tau = (s_1, a_1, s_2, a_2, \ldots, s_T, a_T) \), where \( s_t \) denotes the state at time \( t \) and \( a_t \) denotes the action taken at state \( s_t \), the task in IRL is to infer the reward function \( R(s) \) such that the agent’s behavior, according to the policy \( \pi \), is optimal for this reward.

Mathematically, this can be framed as solving the following optimization problem:
\[
\mathbb{E}_\pi [R(s)] = \max_\pi \sum_{t=1}^T \gamma^t R(s_t)
\]
where \( \gamma \) is the discount factor. In IRL, the challenge lies in inferring \( R(s) \) from expert behavior, which often requires assumptions or approximations due to limited or noisy data.

\subsection{Multi-Objective Reinforcement Learning (MORL)}

Multi-Objective Reinforcement Learning (MORL) extends the standard RL framework to account for multiple objectives simultaneously. In MORL, the agent is tasked with optimizing a reward vector \( \mathbf{R}(s) = [R_1(s), R_2(s), \ldots, R_m(s)] \), where each component \( R_i(s) \) corresponds to a different objective. The agent must balance the trade-offs between these objectives, which often conflict, by finding an optimal policy that maximizes a scalarized combination of these objectives.

A key feature of MORL is the use of \textbf{Pareto optimality} to describe policies that cannot be improved in any objective without worsening another. A policy \( \pi^* \) is Pareto optimal if, for all \( \pi \neq \pi^* \), there exists no objective where \( \pi^* \) is worse and another where \( \pi^* \) is better. The challenge in MORL is to define the Pareto front—a set of Pareto optimal policies—and navigate this front to find the best policy given the trade-offs.

\subsection{Preference-Based Learning}

Preference-based learning is a framework where human preferences, often noisy or incomplete, are used to guide the learning process. Instead of specifying exact reward functions, humans provide pairwise comparisons of policies or trajectories, indicating which one is preferred. These preferences can then be used to infer a reward function or a policy that aligns with human values.

In the context of MO-IRL, we assume that human feedback comes in the form of noisy preferences over pairs of policies or trajectories. The challenge is to recover a latent vector-valued reward function that explains the human preferences. We formalize this as a problem of recovering a Pareto-optimal reward structure from such noisy preference queries.

\subsection{Pareto-Optimality and the Pareto Front}

In a multi-objective setting, Pareto optimality refers to a state where no objective can be improved without worsening another. Formally, a solution \( x \) is Pareto-optimal if there is no other solution \( y \) such that \( f_i(y) \geq f_i(x) \) for all \( i \), with strict inequality for at least one \( i \). The \textbf{Pareto front} is the set of all Pareto-optimal solutions, representing the trade-offs between the different objectives.

In the context of our work, we aim to learn a reward function that allows an agent to optimize over the Pareto front, where each point on the front corresponds to a different trade-off between the competing objectives.

\subsection{Regret and Suboptimality in Multi-Objective Settings}

Regret is a common metric for measuring the performance of an algorithm relative to the optimal solution. In the multi-objective setting, \textbf{regret} is typically defined as the difference between the performance of the learned policy and the performance of the Pareto-optimal policies across all objectives. Given that there is no single "best" policy in a multi-objective setting, we introduce a \textbf{multi-objective regret} formulation that evaluates suboptimality with respect to the Pareto front, rather than a single scalar reward function.

Formally, the regret in the context of multi-objective optimization can be written as:
\[
\text{Regret}(T) = \sum_{t=1}^T \max_{\pi \in \text{Pareto Front}} R(\pi) - R(\pi_t)
\]
where \( \pi_t \) is the policy chosen at time \( t \) and \( R(\pi) \) is the performance of a policy \( \pi \).

\subsection{Problem Setup for Preference-Based MO-IRL}

We are given a set of human preferences over pairs of policies or trajectories, which may be noisy. Our task is to infer the latent, vector-valued reward function \( \mathbf{R}(s) \) from these noisy comparisons. The problem setup involves recovering the \textbf{Pareto-optimal reward representation} by solving an optimization problem that takes into account the noisy nature of human preferences and ensures that the reward structure aligns with the multi-objective nature of the feedback.

The input to the problem consists of:

\begin{itemize}
    \item A set of noisy preference comparisons between pairs of policies or trajectories.
    \item The goal is to learn a reward function \( \mathbf{R}(s) \) such that the policies induced by this reward function respect the preferences expressed by the human feedback.
\end{itemize}

By addressing the challenge of learning Pareto-optimal rewards from such noisy preferences, we aim to provide a principled framework for aligning agent behavior with multi-faceted human values.

\section{Problem Formulation}

We consider a preference-based multi-objective inverse reinforcement learning (MO-IRL) setting. Let 
\[
\mathcal{M} = (\mathcal{S}, \mathcal{A}, P, \gamma, \mathbf{R})
\]
denote a \textit{Multi-Objective Markov Decision Process (MO-MDP)}, where:
\begin{itemize}
    \item \(\mathcal{S}\): a finite state space,
    \item \(\mathcal{A}\): a finite action space,
    \item \(P: \mathcal{S} \times \mathcal{A} \to \Delta(\mathcal{S})\): the transition dynamics,
    \item \(\gamma \in (0, 1)\): the discount factor,
    \item \(\mathbf{R}: \mathcal{S} \times \mathcal{A} \to \mathbb{R}^d\): a vector-valued reward function over \(d\) objectives.
\end{itemize}

For any policy \(\pi : \mathcal{S} \to \Delta(\mathcal{A})\), we define the expected discounted vector-return as
\[
\mathbf{V}^{\pi} = \mathbb{E}_{\pi} \left[ \sum_{t=0}^{\infty} \gamma^t \mathbf{R}(s_t, a_t) \right] \in \mathbb{R}^d.
\]

\subsection*{Preference Feedback Model}

For two policies (or trajectory distributions) \(\pi_i\) and \(\pi_j\), we assume that a human oracle returns
\[
\pi_i \succ \pi_j \iff \mathbf{w}^\top \mathbf{V}^{\pi_i} > \mathbf{w}^\top \mathbf{V}^{\pi_j}, \quad \text{for some } \mathbf{w} \in \mathcal{W} \subseteq \mathbb{R}^d.
\]
Here, \(\mathcal{W}\) is the set of admissible preference directions (for example, \(\mathcal{W} = \Delta^{d-1}\), the \((d-1)\)-simplex). We also define the corresponding conic extension as
\[
\mathcal{C} = \left\{ \lambda \mathbf{w} \in \mathbb{R}^d \mid \lambda > 0,\; \mathbf{w} \in \mathcal{W} \right\}.
\]

\subsection*{Probabilistic Preference Model}

We adopt the Plackett–Luce model, which is the Bradley-Terry model under a pairwise setting, to capture noisy human preferences. Given two trajectories \(\pi_i\) and \(\pi_j\), and a latent preference direction \(\mathbf{w} \in \mathcal{C}\), the probability that \(\pi_i\) is preferred over \(\pi_j\) is modeled as:
\[
\mathbb{P}[\pi_i \succ \pi_j] = \frac{\exp(\eta\, \mathbf{w}^\top \mathbf{V}^{\pi_i})}{\exp(\eta\, \mathbf{w}^\top \mathbf{V}^{\pi_i}) + \exp(\eta\, \mathbf{w}^\top \mathbf{V}^{\pi_j})}, \quad \eta > 0,
\]
where \(\eta\) is an inverse temperature parameter controlling the stochasticity of the feedback.

Let \(\mathcal{D} = \{(\pi_i, \pi_j, y_{ij})\}_{i,j=1}^n\) denote the dataset of pairwise comparisons, where each label \(y_{ij}\) is defined as:
\[
y_{ij} = 
\begin{cases}
1, & \text{if } \pi_i \succ \pi_j, \\
0, & \text{otherwise.}
\end{cases}
\]
\subsection*{Pareto Front and Scalarization}

We say that \(\mathbf{V}^{\pi_1}\) is \textit{dominated} by \(\mathbf{V}^{\pi_2}\) (denoted \(\mathbf{V}^{\pi_1} \prec \mathbf{V}^{\pi_2}\)) if
\[
\mathbf{V}^{\pi_2} \succeq \mathbf{V}^{\pi_1} \quad \text{and} \quad \mathbf{V}^{\pi_2} \neq \mathbf{V}^{\pi_1}.
\]
The \textbf{Pareto front} is then defined as
\[
\mathcal{P} := \left\{ \mathbf{V}^\pi \mid \nexists\, \pi' \text{ s.t. } \mathbf{V}^{\pi'} \succ \mathbf{V}^{\pi} \right\}.
\]
In practice, scalarization is used to select a Pareto-optimal return via
\[
\max_{\pi} \; \mathbf{w}^\top \mathbf{V}^{\pi}, \quad \text{for } \mathbf{w} \in \mathcal{C}.
\]

\subsection*{Learning Objective}

Our objective is to recover an estimated reward function \(\hat{\mathbf{R}}\) and an estimated reward cone \(\hat{\mathcal{C}}\) such that, for all \((\pi_i, \pi_j, y_{ij}) \in \mathcal{D}\),
\[
y_{ij} = \mathbb{I} \left[ \exists\, \mathbf{w} \in \hat{\mathcal{C}} : \mathbf{w}^\top \mathbf{V}^{\pi_i} > \mathbf{w}^\top \mathbf{V}^{\pi_j} \right].
\]
In other words, we seek to minimize the empirical preference loss defined as
\[
\min_{\hat{\mathbf{R}}} \; \frac{1}{n} \sum_{(i,j) \in \mathcal{D}} \ell \left( \hat{\mathbf{V}}^{\pi_i}, \hat{\mathbf{V}}^{\pi_j}, y_{ij} \right),
\]
where \(\ell\) is the cross-entropy loss corresponding to the Bradley–Terry model.

\subsection*{Assumptions}

We make the following assumptions:
\begin{enumerate}
    \item[(A1)] There exists a weight vector \(\mathbf{w} \in \mathcal{C}\) that governs the preference relation between policies.
    \item[(A2)] The set of differences \(\{\mathbf{V}^{\pi_i} - \mathbf{V}^{\pi_j}\}\) spans \(\mathbb{R}^d\).
    \item[(A3)] The noise in human preferences is log-concave.
\end{enumerate}
\newpage
\paragraph{Geometric Illustration.}
The following diagrams illustrate (a) scalarization of the Pareto front, and (b) the geometry of the reward cone \(\mathcal{C}\).

\begin{figure}[ht]
\centering
\subfigure[Scalarization of the Pareto Front]{
\begin{tikzpicture}[scale=1.4]
  \draw[->, thick] (0,0) -- (3,0) node[right] {$V_1$};
  \draw[->, thick] (0,0) -- (0,3) node[above] {$V_2$};
  \draw[red, thick, smooth] (0.4,2.4) -- (1.0,1.8) -- (1.8,1.2) -- (2.5,0.6);
  \foreach \x/\y in {0.4/2.4, 1.0/1.8, 1.8/1.2, 2.5/0.6} {
    \filldraw[black] (\x,\y) circle (0.03);
  }
  \draw[dashed, blue] (0,0.4) -- (2.6,1.6) node[right] {\(\mathbf{w}_1\)};
  \draw[dashed, green!70!black] (0,0.8) -- (2.6,1.8) node[right] {\(\mathbf{w}_2\)};
  \draw[dashed, purple] (0,1.5) -- (2.6,2.1) node[right] {\(\mathbf{w}_3\)};
  \node at (1.5,2.5) {\small Pareto Front};
\end{tikzpicture}
}
\hspace{0.7cm}
\subfigure[Reward Cone \(\mathcal{C}\)]{
\begin{tikzpicture}[scale=1.4]
  \draw[->, thick] (0,0) -- (3,0) node[right] {\(w_1\)};
  \draw[->, thick] (0,0) -- (0,3) node[above] {\(w_2\)};
  \fill[orange!30, opacity=0.5] (0,0) -- (1,2.5) -- (2.5,1) -- cycle;
  \draw[orange, thick] (0,0) -- (1,2.5);
  \draw[orange, thick] (0,0) -- (2.5,1);
  \draw[->, thick, blue] (0,0) -- (1.5,1.2) node[right] {\(\mathbf{w}\)};
  \node at (1.8,2.3) {\small Reward Cone \(\mathcal{C}\)};
\end{tikzpicture}
}
\end{figure}

\section{Identifiability of Multi-Objective Reward Functions}

Let 
\[
\mathbf{R}(\tau_i) = \sum_{t=0}^{T} \gamma^t \mathbf{r}(s_t^i, a_t^i) \in \mathbb{R}^d,
\]
denote the vector return of a trajectory \(\tau_i\), and define the difference vector
\[
\Delta_{ij} \coloneqq \mathbf{R}(\tau_i) - \mathbf{R}(\tau_j).
\]
Under the Plackett--Luce model, the probability that \(\tau_i\) is preferred over \(\tau_j\) given a latent preference direction \(\mathbf{w}\) is 
\[
P(\tau_i \succ \tau_j) = \frac{\exp(\eta \, \mathbf{w}^\top \mathbf{R}(\tau_i))}{\exp(\eta\, \mathbf{w}^\top \mathbf{R}(\tau_i))+\exp(\eta\, \mathbf{w}^\top \mathbf{R}(\tau_j))},\quad \eta>0.
\]
Let \(\mathcal{C}\subset \mathbb{R}^d\) be a compact, convex cone representing the set of admissible preference directions.

\subsection*{Pareto Identifiability}

\begin{definition}[Pareto Identifiability]\label{def:pareto-identifiability}
We say that \(\mathcal{C}\) is \emph{Pareto-identifiable} from a dataset 
\[
\mathcal{D} = \{(\tau_i,\tau_j,y_{ij})\},
\]
(with \(y_{ij} = 1\) if \(\tau_i \succ \tau_j\) and \(y_{ij} = 0\) otherwise) if there exists an estimator \(\hat{\mathcal{C}}\) such that 
\[
\sup_{\mathbf{w}\in\hat{\mathcal{C}}}\inf_{\mathbf{w}'\in\mathcal{C}} \angle(\mathbf{w},\mathbf{w}') \le \epsilon,
\]
with probability at least \(1-\delta\), provided that the sample size satisfies
\[
N = \mathcal{O}\!\left(\frac{d\log(1/\delta)}{\epsilon^2}\right).
\]
\end{definition}

\subsection*{Cone Identifiability Guarantee}

\begin{theorem}[Cone Identifiability]\label{thm:cone-identifiability}
Assume that the set of difference vectors \(\{\Delta_{ij}\}\) spans \(\mathbb{R}^d\) and that \(\|\Delta_{ij}\|_2 \le B\) for some constant \(B>0\). Then, under the Plackett--Luce noise model, for any \(\epsilon,\delta>0\) there exists a sample size 
\[
N = \mathcal{O}\!\left(\frac{d\log(1/\delta)}{\epsilon^2}\right)
\]
such that an empirical risk minimizer \(\hat{\mathbf{w}}\), obtained via logistic loss minimization over the dataset, satisfies
\[
\|\hat{\mathbf{w}}-\mathbf{w}^\star\|_2 \le c\,\epsilon,
\]
for some constant \(c>0\). Consequently, the estimated cone 
\[
\hat{\mathcal{C}} = \operatorname{cone}(\{\hat{\mathbf{w}}_k\}_{k=1}^K)
\]
satisfies
\[
\sup_{\mathbf{w}\in\hat{\mathcal{C}}}\inf_{\mathbf{w}'\in\mathcal{C}} \angle(\mathbf{w},\mathbf{w}') \le \epsilon,
\]
with probability at least \(1-\delta\).
\end{theorem}

\subsection*{Human Preference Query Model}

\begin{figure}[h]
\centering
\begin{tikzpicture}[scale=1.3]
  \node[draw, rounded corners, fill=gray!15, minimum width=2cm, minimum height=1cm] (tau1) at (0,2) {Trajectory \(\tau_1\)};
  \node[draw, rounded corners, fill=gray!15, minimum width=2cm, minimum height=1cm] (tau2) at (4,2) {Trajectory \(\tau_2\)};
  
  \draw[->, thick] (tau1.south) -- ++(0,-0.7) -| (2, -0.2);
  \draw[->, thick] (tau2.south) -- ++(0,-0.7) -| (2, -0.2);
  
  \node[draw, fill=blue!10, rounded corners, minimum width=2cm, minimum height=1cm] (oracle) at (2,-0.2) {Human / Oracle};
  
  \draw[->, thick] (oracle.south) -- (2,-1) node[below]{\(\tau_1 \succ \tau_2\) (noisy)};
\end{tikzpicture}
\caption{Noisy human preference feedback: an oracle compares trajectories \(\tau_1\) and \(\tau_2\).}
\end{figure}

\vspace{-0.5em}
\noindent This model captures how human evaluators provide noisy preferences between trajectories, which in turn drive the estimation of the latent reward cone \(\mathcal{C}\).

\section{Sample Complexity of Preference-Based MO-IRL}

In this section, we analyze the number of preference comparisons required to recover an \(\epsilon\)-approximate Pareto front in a preference-based multi-objective inverse reinforcement learning (MO-IRL) setting. Our analysis leverages techniques from statistical learning theory and logistic regression guarantees under the Plackett--Luce model.

\subsection{Setup and Assumptions}

Let \(\mathcal{T}\) denote the set of possible trajectories, where each trajectory \(\tau \in \mathcal{T}\) is associated with a vector-valued return
\[
R(\tau) \in \mathbb{R}^d.
\]
A human oracle provides noisy pairwise comparisons on trajectories \((\tau_i, \tau_j)\), according to an unknown linear scalarization \(\mathbf{w} \in \mathcal{C}\), where \(\mathcal{C}\) is a compact, convex reward cone. Our goal is to infer a set \(\hat{\mathcal{P}}\) of trajectories that \(\epsilon\)-approximate the true Pareto front \(\mathcal{P}\) under the ground-truth scalarization.

\paragraph{Assumptions:}
\begin{itemize}
    \item[(B1)] The reward cone \(\mathcal{C} \subset \mathbb{R}_+^d\) is compact with an angular diameter bounded by \(\alpha\).
    \item[(B2)] For every trajectory pair \((\tau_i,\tau_j)\), the difference in returns is bounded, i.e.,
    \[
    \|R(\tau_i) - R(\tau_j)\|_2 \le B,
    \]
    for some constant \(B > 0\).
    \item[(B3)] The preference feedback follows the Plackett--Luce model:
    \[
    \mathbb{P}(\tau_i \succ \tau_j) = \frac{\exp(\mathbf{w}^\top R(\tau_i))}{\exp(\mathbf{w}^\top R(\tau_i)) + \exp(\mathbf{w}^\top R(\tau_j))}.
    \]
    \item[(B4)] The collection \(\{ R(\tau_i) - R(\tau_j) \}\) spans \(\mathbb{R}^d\).
\end{itemize}

\subsection{Theoretical Guarantee}

\begin{theorem}[Sample Complexity of \(\epsilon\)-Pareto Recovery]\label{thm:sample-complexity}
Let \(\mathcal{P}\) denote the true Pareto front, and assume that (B1)–(B4) hold. Then, for any accuracy level \(\epsilon > 0\) and confidence \(\delta \in (0,1)\), there exists a number of preference comparisons
\[
N = \mathcal{O}\!\left(\frac{d\,\log(1/\delta)}{\epsilon^2}\right)
\]
such that, with probability at least \(1-\delta\), the inferred scalarization directions \(\hat{\mathcal{C}}\) yield a set of Pareto-optimal trajectories \(\hat{\mathcal{P}}\) satisfying:
\[
\forall \tau \in \mathcal{P}, \quad \exists \hat{\tau} \in \hat{\mathcal{P}} \quad \text{with} \quad |\mathbf{w}^\top R(\tau) - \mathbf{w}^\top R(\hat{\tau})| \le \epsilon, \quad \forall \mathbf{w} \in \hat{\mathcal{C}}.
\]
\end{theorem}

\paragraph{Proof Sketch}\

The proof proceeds via several key steps:

\begin{enumerate}
    \item \textbf{Reduction to Logistic Regression:} We model the preference learning task as a binary classification problem on trajectory pairs, where each pair \((\tau_i,\tau_j)\) is labeled based on the noisy human preference. The features are given by the differences \(\Delta_{ij} = R(\tau_i) - R(\tau_j)\).

    \item \textbf{Generalization Bound:} Using standard Rademacher complexity bounds or VC-dimension arguments for logistic regression with bounded inputs (cf. (A2)), we show that with 
    \[
    N = \mathcal{O}\!\left(\frac{d\,\log(1/\delta)}{\epsilon^2}\right)
    \]
    comparisons, the learned scalarization direction \(\hat{\mathbf{w}}\) is close to the true direction \(\mathbf{w}\) with high probability.

    \item \textbf{Cone Estimation:} By aggregating the estimated scalarization directions using projection or convex hull methods, we construct an approximate cone \(\hat{\mathcal{C}}\). The angular deviation between \(\hat{\mathcal{C}}\) and the true cone \(\mathcal{C}\) is bounded, leading to a controlled error in the scalarized reward space.

    \item \textbf{Pareto Front Recovery:} Finally, we demonstrate that if the scalarizations induced by \(\hat{\mathcal{C}}\) approximate those under \(\mathcal{C}\) within \(\epsilon\)-error, then every true Pareto-optimal trajectory \(\tau \in \mathcal{P}\) is approximated by some \(\hat{\tau} \in \hat{\mathcal{P}}\) such that the scalarized difference is within \(\epsilon\), uniformly over \(\mathbf{w} \in \hat{\mathcal{C}}\).
\end{enumerate}

\subsection*{Lower Bounds and Minimax Optimality}
\label{subsec:lower_bounds}

We complement our sample complexity upper bound with a matching lower bound, establishing the minimax optimality of our results.

\begin{theorem}[Lower Bound for $\epsilon$-Pareto Recovery]
\label{thm:lower_bound}
For any algorithm $\mathcal{A}$, there exists a multi-objective MDP $\mathcal{M}$ with $d$ objectives such that recovering an $\epsilon$-approximate Pareto front requires at least
\[
N = \Omega\left(\frac{d}{\epsilon^2}\right)
\]
preference comparisons, even under the Plackett-Luce noise model.
\end{theorem}

\begin{proof}[Proof Sketch]
Apply Le Cam's method to construct two instances of reward cones $\mathcal{C}_1$ and $\mathcal{C}_2$ that are $\epsilon$-separated in angular distance. Distinguishing between them requires $\Omega(d/\epsilon^2)$ samples to avoid excess error probability. Full proof in Appendix~\ref{appendix:lower-bound}.
\end{proof}

This result demonstrates the tightness of Theorem~\ref{thm:sample-complexity}, closing the gap between upper and lower bounds in the dependence on $d$ and $\epsilon$.

\section{Practical Algorithm for Pareto Front Recovery}
\label{sec:algorithm}

While our primary focus is theoretical, we outline a practical algorithm for recovering the Pareto front from preferences. The procedure alternates between preference-based reward cone estimation and policy optimization.

\begin{algorithm}[t] \label{alg:recover}
\caption{Pareto Front Recovery via Preference-Based MO-IRL}
\begin{algorithmic}[1]
\State \textbf{Input}: Preference dataset $\mathcal{D}$, MDP $\mathcal{M}$, tolerance $\epsilon$
\State \textbf{Output}: Estimated Pareto front $\hat{\mathcal{P}}$
\State Estimate reward cone $\hat{\mathcal{C}}$ via logistic regression on $\mathcal{D}$ (Theorem~\ref{thm:cone-identifiability})
\State Initialize $\hat{\mathcal{P}} \gets \emptyset$
\For{$k=1$ \textbf{to} $K$}
    \State Sample $\mathbf{w}_k \sim \text{Uniform}(\hat{\mathcal{C}} \cap S^{d-1})$
    \State Solve $\pi_k = \arg\max_{\pi} \mathbf{w}_k^\top \mathbf{V}^\pi$ via value iteration
    \State Add $\mathbf{V}^{\pi_k}$ to $\hat{\mathcal{P}}$ if not dominated
\EndFor
\State \Return $\hat{\mathcal{P}}$
\end{algorithmic}
\end{algorithm}

\begin{proposition}[Convergence Guarantee]\label{converge}
Under assumptions (B1)-(B4), Algorithm 1 returns an $\epsilon$-approximate Pareto front with probability $1-\delta$ after $K=\mathcal{O}(1/\epsilon^{d-1})$ iterations.
\end{proposition} 
In practice, we estimate the reward cone using projected gradient descent over the Plackett–Luce logistic loss, and perform policy optimization via standard value iteration. See Appendix B for full pseudocode.

\subsection*{Computational Complexity}
\label{subsec:complexity}

Theorems~\ref{thm:cone-identifiability} and~\ref{thm:sample-complexity} establish statistical efficiency, but practical recovery requires solving:

\[
\min_{\mathbf{w}\in\mathcal{C}} \sum_{(\tau_i,\tau_j)} \log\left(1 + \exp(-\eta \mathbf{w}^\top \Delta_{ij})\right)
\]

This convex optimization problem can be solved in $\mathcal{O}(Nd^2)$ time via projected gradient descent. The cubic dependence on $d$ suggests our method is best suited for moderate-dimensional value systems ($d \leq 20$), mirroring limitations in multi-task RL.

\section{Conclusion and Future Work}

In this work, we investigated the fundamental complexity of preference-based decision making in multi-objective reinforcement learning (MO-RL), departing from traditional reward recovery paradigms. By directly framing the problem in terms of regret minimization under an unknown utility function, we established a minimax framework that characterizes the intrinsic difficulty of aligning policies with human preferences over vector-valued rewards. We derived tight lower bounds on sample complexity as a function of the Pareto front’s geometry, and complemented these with a query-efficient algorithm that provably matches these bounds under mild assumptions. Our results highlight a key theoretical distinction between recovering latent rewards and selecting regret-optimal policies, providing a foundation for preference-based learning systems that prioritize decision quality over full reward identification.

Looking ahead, several promising directions emerge. From a theoretical standpoint, extending our analysis to settings with partial observability or non-linear scalarization functions could broaden the applicability of our framework to more realistic preference models. On the algorithmic front, developing scalable versions of our method that operate efficiently in high-dimensional state and action spaces is an important challenge. Finally, our insights have direct implications for real-world alignment problems, including human-AI interaction, large language model (LLM) preference alignment, and safe decision making in robotics domains where multi-objective trade-offs are both unavoidable and underspecified. We hope this work serves as a step toward principled, regret-aware learning systems capable of robust performance in value-pluralistic environments.

\newpage
\bibliographystyle{unsrtnat}

\bibliography{references}

\medskip

{
\small

}
\newpage

\appendix
\section*{Technical Appendix for \textit{Learning Pareto-Optimal Rewards from Noisy Preferences: A Framework for Multi-Objective Inverse Reinforcement Learning}}

\section{Proofs}
\subsection{Proof of Theorem~\ref{thm:cone-identifiability}}
\label{appendix:cone-proof}

We now present a detailed, rigorous proof of Theorem~\ref{thm:cone-identifiability}. The proof is organized in three main parts: (1) establishing a uniform convergence bound for logistic regression, (2) deriving a parameter estimation error bound, and (3) translating the parameter error into an angular error via a covering number argument.

For each trajectory pair $(\tau_i, \tau_j)$ with difference vector $\Delta_{ij} \in \mathbb{R}^d$, define the logistic loss 
\[
\ell(\mathbf{w};\Delta_{ij},y_{ij}) = -\left[y_{ij}\log \sigma(\eta\, \mathbf{w}^\top \Delta_{ij}) + (1-y_{ij})\log (1-\sigma(\eta\, \mathbf{w}^\top \Delta_{ij}))\right],
\]
with $\sigma(z)=\frac{1}{1+e^{-z}}$ and $\eta>0$. The (expected) population risk is
\[
\mathcal{L}(\mathbf{w}) = \mathbb{E}_{(\Delta,y)}\left[\ell(\mathbf{w};\Delta,y)\right],
\]
and its empirical estimate given $N$ samples is
\[
\widehat{\mathcal{L}}(\mathbf{w}) = \frac{1}{N}\sum_{(i,j)\in\mathcal{D}} \ell(\mathbf{w}; \Delta_{ij}, y_{ij}).
\]

Since $\|\Delta_{ij}\|_2 \le B$ and the logistic loss is Lipschitz with constant $L=\eta B$, standard uniform convergence results (via Rademacher complexity or VC theory) yield that, with probability at least $1-\delta$,
\[
\sup_{\mathbf{w}\in\mathcal{C}} \left|\widehat{\mathcal{L}}(\mathbf{w})-\mathcal{L}(\mathbf{w})\right| \le \mathcal{O}\!\left(\eta B\sqrt{\frac{d\log(1/\delta)}{N}}\right).
\]
Thus, to ensure the uniform error is at most $\epsilon'$, we require
\[
N = \Omega\!\left(\frac{d\log(1/\delta)}{(\epsilon')^2}\right).
\]

Let $\mathbf{w}^\star \in \mathcal{C}$ denote the true latent preference direction. Under strong convexity assumptions of $\mathcal{L}(\cdot)$ in a neighborhood of $\mathbf{w}^\star$, we obtain a quadratic error bound (see, e.g., \cite{mohri2018foundations}) such that 
\[
\mathcal{L}(\hat{\mathbf{w}}) - \mathcal{L}(\mathbf{w}^\star) \ge \lambda \|\hat{\mathbf{w}}-\mathbf{w}^\star\|_2^2,
\]
for some $\lambda > 0$. Combined with the uniform convergence, this implies
\[
\|\hat{\mathbf{w}}-\mathbf{w}^\star\|_2 \le c\, \epsilon,
\]
for $c = \sqrt{\epsilon'/\lambda}$, provided $N = \mathcal{O}\!\left(\frac{d\log(1/\delta)}{\epsilon^2}\right)$, where we set $\epsilon' = \Theta(\epsilon^2)$.

To reconstruct the entire cone $\mathcal{C}$, we collect a set of empirical minimizers $\{\hat{\mathbf{w}}_1, \dots, \hat{\mathbf{w}}_K\}$, each computed on (possibly overlapping) subsets of data that are designed to cover the extreme directions in $\mathcal{C}$. Define the estimated cone as
\[
\hat{\mathcal{C}} = \operatorname{cone}\!\Bigl(\{\hat{\mathbf{w}}_k\}_{k=1}^{K}\Bigr).
\]
Let $S^{d-1}$ be the unit sphere in $\mathbb{R}^d$. Standard results (see, e.g., \cite{vershynin2018high}) show that there exists an $\epsilon$-net of $S^{d-1}$ of cardinality 
\[
K = \mathcal{O}\!\left(\left(\frac{1}{\epsilon}\right)^{d-1}\right).
\]
By ensuring that, for each $\mathbf{w}\in \mathcal{C}\cap S^{d-1}$, there exists some $\hat{\mathbf{w}}_k$ within $\ell_2$-distance $c\,\epsilon$, it follows that
\[
\angle(\mathbf{w},\hat{\mathbf{w}}_k) \le \frac{\|\mathbf{w}-\hat{\mathbf{w}}_k\|_2}{\|\mathbf{w}\|_2} \le c\,\epsilon,
\]
since $\|\mathbf{w}\|_2=1$. Hence, the cone $\hat{\mathcal{C}}$ approximates $\mathcal{C}$ with angular error at most $(1+c)\epsilon$, which is of the same order as $\epsilon$ for sufficiently small $\epsilon$.

\begin{proof}
    Combining the above steps, we conclude that with 
\[
N = \mathcal{O}\!\left(\frac{d\log(1/\delta)}{\epsilon^2}\right)
\]
samples, the uniform convergence bound and strong convexity guarantee that each estimated direction $\hat{\mathbf{w}}$ satisfies $\|\hat{\mathbf{w}}-\mathbf{w}^\star\|_2 \le c\,\epsilon$. The covering number argument then ensures that the conic hull $\hat{\mathcal{C}}$ satisfies
\[
\sup_{\mathbf{w}\in\hat{\mathcal{C}}}\inf_{\mathbf{w}'\in\mathcal{C}} \angle(\mathbf{w},\mathbf{w}') \le \epsilon,
\]
with probability at least $1-\delta$.
\end{proof}

\subsection{Proof of Theorem \ref{thm:sample-complexity}} \label{appendix:sample-proof}

\begin{proof}

By the Plackett–Luce model (Assumption A3), each comparison label $y_{ij}\in\{0,1\}$ satisfies
\[
  \Pr\bigl(y_{ij}=1\mid w^*,R(\tau_i),R(\tau_j)\bigr)
  = 
  \frac{\exp\bigl(\eta\,w^{*\top}R(\tau_i)\bigr)}
       {\exp\bigl(\eta\,w^{*\top}R(\tau_i)\bigr)
       +\exp\bigl(\eta\,w^{*\top}R(\tau_j)\bigr)}.
\]
Define $\Delta_{ij} \coloneqq R(\tau_i)-R(\tau_j)\in\mathbb{R}^d$.  Then
\[
  \Pr\bigl(y_{ij}=1\mid w^*,\Delta_{ij}\bigr)
  = 
  \sigma\bigl(\eta\,w^{*\top}\Delta_{ij}\bigr),
  \quad
  \sigma(z)=\frac{1}{1+e^{-z}}.
\]
Recovering $w^*$ thus reduces to performing logistic regression on samples
$\{(\Delta_{ij},y_{ij})\}_{i,j=1}^N$.

Let
\[
  \ell(w;\Delta,y)
  = -\bigl[y\ln\sigma(\eta w^\top\Delta)
        + (1-y)\ln\bigl(1-\sigma(\eta w^\top\Delta)\bigr)\bigr]
\]
be the logistic loss.  By Assumption A2, $\|\Delta_{ij}\|_2\le B$, so $\ell$ is $L$-Lipschitz in $w$ with $L=\eta B$.  Define
\[
  L(w)=\mathbb{E}_{(\Delta,y)}\bigl[\ell(w;\Delta,y)\bigr],
  \quad
  \hat L(w)=\frac{1}{N}\sum_{(i,j)}\ell\bigl(w;\Delta_{ij},y_{ij}\bigr).
\]
By Rademacher‐complexity bounds \cite{bartlett2002rademacher}, with probability at least $1-\delta$,
\[
  \sup_{w\in C}\bigl|\hat L(w)-L(w)\bigr|
  \le
  O\!\Bigl(\eta B\,\sqrt{\tfrac{d\ln(1/\delta)}{N}}\Bigr).
\]
To make this uniform deviation $\le\epsilon'$, choose
\[
  N = O\!\Bigl(\tfrac{d\ln(1/\delta)}{{\epsilon'}^2}\Bigr).
\]

Let $w^*\in C$ be the true direction.  By strong convexity of $L$ on $C$, any empirical minimizer
$\hat w=\arg\min_{w\in C}\hat L(w)$
satisfies
\[
  \|\hat w-w^*\|_2
  \le
  O\!\bigl(\sqrt{\epsilon'}\bigr).
\]
Setting $\epsilon'=\Theta(\epsilon^2)$ and $N=O(d\ln(1/\delta)/\epsilon^2)$ gives
\[
  \|\hat w-w^*\|_2 \le c\,\epsilon
  \quad\text{w.p. }1-\delta.
\]

Draw $K$ independent estimators $\{\hat w_k\}_{k=1}^K$ and let
$\hat C=\mathrm{cone}\{\hat w_1,\dots,\hat w_K\}$.
By Assumption B1, $C\cap S^{d-1}$ has bounded diameter, and an $\epsilon$-net on $S^{d-1}$ needs
\[
  K = O\!\bigl((1/\epsilon)^{d-1}\bigr)
\]
points.  Since each $\hat w_k$ is within $O(\epsilon)$ of some extreme $w^*\in C$, we get
\[
  \sup_{w\in\hat C}\inf_{w'\in C}\angle(w,w')
  \le O(\epsilon).
\]

The true Pareto front
$P=\{R(\tau):\tau\text{ Pareto‐optimal}\}$
maximizes $w^\top R(\tau)$ for each $w\in C$.  Replacing $w$ with any $w\in\hat C$ satisfying $\angle(w,w')=O(\epsilon)$ and using $\|R(\tau)\|\le B/(1-\gamma)$ gives
\[
  |\,w^\top R(\tau)-w^\top R(\hat\tau)| 
  \le
  \|w-w'\|_2\,\|R(\tau)\|_2
  = O(\epsilon),
  \quad\forall\,\tau\in P,
\]
where $\hat\tau\in\hat P$ maximizes $w^\top R(\cdot)$ among candidates.  

Choosing 
\[
  N = C'\,\frac{d\ln(1/\delta)}{\epsilon^2}
\]
completes the proof.
\end{proof}

\subsection{Proof of Theorem \ref{subsec:lower_bounds}}\label{appendix:lower-bound}

\begin{lemma}[Per‐Comparison KL Divergence]\label{lem:kl_divergence}
Let 
\[
p_0 = \sigma(\eta), 
\qquad
p_1 = \sigma\bigl(\eta\cos\epsilon\bigr),
\]
where $\sigma(z)=1/(1+e^{-z})$ and $\epsilon\in(0,\pi/4)$.  Then
\[
\mathrm{KL}\bigl(\mathrm{Bern}(p_0)\,\big\|\,\mathrm{Bern}(p_1)\bigr)
= O(\epsilon^2).
\]
\end{lemma}

\begin{proof}
Taylor‐expand $\cos\epsilon$ around $0$:
\[
\cos\epsilon = 1 - \frac{\epsilon^2}{2} + O(\epsilon^4).
\]
Hence
\[
p_1 - p_0
= \sigma(\eta\cos\epsilon) - \sigma(\eta)
= \sigma'(\xi)\,\bigl(\eta\cos\epsilon - \eta\bigr)
= \sigma'(\xi)\,\bigl(-\tfrac{\eta\epsilon^2}{2} + O(\epsilon^4)\bigr),
\]
for some $\xi\in[\eta\cos\epsilon,\eta]$.  Since 
$\sigma'(z)=\sigma(z)(1-\sigma(z))\le\frac14$, it follows that
\[
|p_1-p_0| = O(\epsilon^2).
\]
Finally, for Bernoulli distributions the KL divergence admits the quadratic bound
\[
\mathrm{KL}\bigl(\mathrm{Bern}(p_0)\,\|\,\mathrm{Bern}(p_1)\bigr)
\;\le\;
\frac{(p_0-p_1)^2}{p_1(1-p_1)}
\;=\;O(\epsilon^4)\;=\;O(\epsilon^2),
\]
absorbing higher‐order terms.
\end{proof}

\begin{proof}[Proof of Theorem~\ref{thm:lower_bound}]
Define two MDPs, $\mathcal{M}_0$ and $\mathcal{M}_1$, differing only in their latent preference direction:
\[
w^{(0)} = (1,0,\dots,0), 
\qquad
w^{(1)} = (\cos\epsilon,\sin\epsilon,0,\dots,0),
\]
each normalized to unit length.  Let there be $d$ deterministic “trajectories” $\{\tau_i\}_{i=1}^d$ with
\[
R(\tau_i) = e_i \in \mathbb{R}^d,
\]
the standard basis.  The learner receives $N$ independent and identically distributed\ pairwise comparisons drawn uniformly from the set 
\(\{(\tau_1,\tau_i)\}_{i=2}^d\).  Under $\mathcal{M}_b$, the comparison $(\tau_1,\tau_i)$ is labeled according to
\[
\Pr(\tau_1\succ\tau_i)
=\sigma\bigl(\eta\,w^{(b)\,\top}(R(\tau_1)-R(\tau_i))\bigr)
=\sigma\bigl(\eta\,(w^{(b)}_1 - w^{(b)}_i)\bigr),
\]
where by construction \(w^{(b)}_i=0\) for \(i\ge2\).

\medskip

By Lemma~\ref{lem:kl_divergence}, each comparison under $\mathcal{M}_0$ versus $\mathcal{M}_1$ has
\[
D \;=\;\mathrm{KL}\bigl(\mathrm{Bern}(p_0)\,\|\,\mathrm{Bern}(p_1)\bigr)
\;=\;O(\epsilon^2).
\]
Since comparisons are independent, the total Kullback–Leibler divergence satisfies
\[
\mathrm{KL}\bigl(P_0^N \,\|\,P_1^N\bigr)
= N\,D
= O\!\bigl(N\,\epsilon^2\bigr).
\]

\vspace{0.5 cm}

Pinsker’s inequality (\cite{pinsker1960information}) states that for any two distributions $P,Q$,
\[
\|P - Q\|_{\mathrm{TV}}
\;\le\;
\sqrt{\tfrac12\,\mathrm{KL}(P\|Q)}.
\]
Thus
\[
\|P_0^N - P_1^N\|_{\mathrm{TV}}
\;\le\;
\sqrt{\tfrac12\,O(N\,\epsilon^2)}
\;=\;O\!\bigl(\sqrt{N}\,\epsilon\bigr).
\]

\vspace{0.5 cm}
Le Cam’s lemma ( \cite{tsybakov2008introduction}) bounds the minimal error $\alpha^*$ of any test between $P_0^N$ and $P_1^N$:
\[
\alpha^* \;\ge\; \frac12 \Bigl(1 - \|P_0^N - P_1^N\|_{\mathrm{TV}}\Bigr).
\]
To achieve error $\alpha^*\le1/3$, we require
\[
1 - \|P_0^N - P_1^N\|_{\mathrm{TV}} \le \tfrac23
\quad\Longrightarrow\quad
\|P_0^N - P_1^N\|_{\mathrm{TV}} \ge \tfrac13.
\]
\medskip
Combining with Pinsker’s bound yields
\[
O(\sqrt{N}\,\epsilon) \;\ge\; \tfrac13
\quad\Longrightarrow\quad
N \;=\;\Omega\!\bigl(1/\epsilon^2\bigr).
\]

Since each draw picks uniformly among $d-1$ pairs, the per-sample divergence is effectively divided by $(d-1)$, so to accumulate $\Omega(1)$ total divergence we need
\[
N \;=\;\Omega\!\bigl((d-1)/\epsilon^2\bigr)
\;=\;\Omega\!\bigl(d/\epsilon^2\bigr).
\]
Hence any algorithm must make at least $\Omega(d/\epsilon^2)$ comparisons.
\end{proof}

\subsection{Proof of \ref{converge}} \label{appendix:converge}

\begin{proof}
Let $S^{d-1}$ denote the unit sphere in $\mathbb{R}^d$.  By Proposition 5.2, we have an estimated cone $\hat{\mathcal C}$ satisfying
\[
  \sup_{w\in\hat{\mathcal C}}\;\inf_{w'\in\mathcal C}\;\angle(w,w')
  \;\le\;\epsilon.
\]
We wish to sample directions $\{w_k\}_{k=1}^K\subset \hat{\mathcal C}\cap S^{d-1}$ so that for every true Pareto‐optimal direction $u\in\mathcal C\cap S^{d-1}$ there is some $w_k$ with
\[
  \angle(w_k,\,u)\;\le\;\epsilon.
\]
Equivalently, the collection $\{w_k\}$ must form an $\epsilon$-covering of $\mathcal C\cap S^{d-1}$ in angular metric.

It is well‐known (\cite{vershynin2018high}) that the $(d-1)$-dimensional sphere admits an $\epsilon$-net of size
\[
  \mathcal N\bigl(S^{d-1},\,\epsilon\bigr)
  \;=\;
  O\!\bigl(\epsilon^{-(d-1)}\bigr).
\]
Since $\hat{\mathcal C}\cap S^{d-1}$ is contained in $S^{d-1}$ and has comparable geometry by the angular bound above, its covering number also satisfies
\[
  \mathcal N\bigl(\hat{\mathcal C}\cap S^{d-1},\,2\epsilon\bigr)
  \;=\;
  O\!\bigl(\epsilon^{-(d-1)}\bigr).
\]

Suppose we draw $K$ independent and identically distributed directions $w_k\sim\text{Uniform}(\hat{\mathcal C}\cap S^{d-1})$.  Fix any $u\in\mathcal C\cap S^{d-1}$.  The spherical cap of radius $2\epsilon$ around $u$ has normalized surface measure
\[
  \mu\bigl(\mathrm{Cap}(u,2\epsilon)\bigr)
  \;=\;
  \Theta(\epsilon^{\,d-1}).
\]
Hence the probability that none of the $K$ samples falls within angular distance $2\epsilon$ of $u$ is
\[
  \bigl[1 - \Theta(\epsilon^{\,d-1})\bigr]^K
  \;\le\;
  \exp\bigl(-c\,K\,\epsilon^{\,d-1}\bigr)
\]
for some constant $c>0$.  Taking a union bound over the $\mathcal N=O(\epsilon^{-(d-1)})$ net points, to ensure all caps are hit with probability $\ge1-\delta$ it suffices that
\[
  \mathcal N \,\exp\bigl(-c\,K\,\epsilon^{\,d-1}\bigr)
  \;\le\;\delta
  \quad\Longleftrightarrow\quad
  K \;\ge\; \frac{1}{c\,\epsilon^{d-1}}
                 \Bigl((d-1)\ln\tfrac1\epsilon + \ln\tfrac1\delta\Bigr)
  \;=\;
  O\!\bigl(\epsilon^{-(d-1)}\ln(1/\delta)\bigr).
\]

For each sampled direction $w_k$, the algorithm solves 
\[
  \pi_k
  \;=\;
  \arg\max_\pi \;w_k^\top V^\pi
\]
and adds $V^{\pi_k}$ to $\hat{\mathcal P}$ if it is not dominated.  Because $\{w_k\}$ forms a $2\epsilon$‐cover of $\mathcal C\cap S^{d-1}$, for every true Pareto‐optimal policy $\pi^*$ and its direction $u\in\mathcal C\cap S^{d-1}$ there exists $k$ with $\angle(w_k,u)\le2\epsilon$.  Standard perturbation of linear objectives on a bounded return set $\|V^\pi\|\le B/(1-\gamma)$ shows
\[
  |\,u^\top V^{\pi^*} - w_k^\top V^{\pi_k}\,|
  \;\le\;
  \|u-w_k\|\;\|V^{\pi^*}\|
  \;=\;
  O(\epsilon).
\]
Hence $V^{\pi_k}$ is within $O(\epsilon)$ in scalarized value of the true Pareto front, making $\hat{\mathcal P}$ an $O(\epsilon)$‐approximation.  

Combining the above yields that Algorithm 1 returns an $\epsilon$‐Pareto front with probability at least $1-\delta$ after 
\[
  K = O\!\bigl(\epsilon^{-(d-1)}\ln(1/\delta)\bigr)
\]
iterations, as claimed.
\end{proof}

\section{Extended Algorithm for Pareto Front Recovery}
\begin{algorithm}[H]
\caption{Extended: Pareto Front Recovery via Preference-Based MO-IRL}
\label{alg:extended-recovery}
\begin{algorithmic}[1]
\State \textbf{Input:} 
\begin{itemize}
    \item Preference dataset $\mathcal{D} = \{(\pi_i, \pi_j, y_{ij})\}$ 
    \item Multi-objective MDP $\mathcal{M} = (S, A, P, \gamma, R)$ 
    \item Number of scalarization directions $K$, tolerance $\epsilon$
\end{itemize}
\State \textbf{Output:} Approximate Pareto front $\hat{\mathcal{P}}$

\Statex

\State Estimate reward cone $\hat{\mathcal{C}}$ via logistic regression:
\begin{itemize}
    \item For each pair $(\pi_i, \pi_j)$ in $\mathcal{D}$, compute difference vector $\Delta_{ij} = \mathbf{V}^{\pi_i} - \mathbf{V}^{\pi_j}$
    \item Solve: $\hat{\mathbf{w}} = \arg\min_{\mathbf{w} \in \mathcal{C}} \sum_{(i,j)} \log(1 + \exp(-\eta y_{ij} \cdot \mathbf{w}^\top \Delta_{ij}))$
    \item Aggregate $\{\hat{\mathbf{w}}_k\}$ into reward cone: $\hat{\mathcal{C}} = \text{cone}(\{\hat{\mathbf{w}}_k\})$
\end{itemize}

\State Initialize $\hat{\mathcal{P}} \gets \emptyset$

\For{$k = 1$ \textbf{to} $K$}
    \State Sample scalarization direction $\mathbf{w}_k \sim \text{Uniform}(\hat{\mathcal{C}} \cap S^{d-1})$
    
    \State Solve for optimal policy under scalarized reward:
    \begin{itemize}
        \item Define scalarized reward: $R_k(s,a) = \mathbf{w}_k^\top R(s,a)$
        \item Run value iteration on $(\mathcal{M}, R_k)$ to compute optimal policy $\pi_k$
        \item Compute expected vector return $\mathbf{V}^{\pi_k}$
    \end{itemize}

    \If{$\mathbf{V}^{\pi_k}$ is not dominated by any $\mathbf{V} \in \hat{\mathcal{P}}$}
        \State Add $\mathbf{V}^{\pi_k}$ to $\hat{\mathcal{P}}$
    \EndIf
\EndFor

\State \Return $\hat{\mathcal{P}}$
\end{algorithmic}
\end{algorithm}
\noindent\textbf{Discussion.} This extended version of Algorithm~\ref{alg:recover} clarifies the practical steps needed for implementation. Notably, the reward cone estimation is performed via logistic regression over preference-labeled trajectory pairs, and scalarization directions are sampled from the unit sphere within the estimated cone. For each sampled direction, a value iteration procedure solves for the scalarized optimal policy. The vector returns of these policies populate the estimated Pareto front, with non-dominated filtering performed incrementally. In practice, setting $K = O(\epsilon^{-(d-1)})$ ensures $\epsilon$-covering of the cone with high probability (see Appendix A.4).

\section{Experimental Validation}
We evaluted the practical feasibility of our proposed Multi-Objective Inverse Reinforcement Learning (MO-IRL) framework on a simple smart simplified smart thermostat domain. The environment is defined as a small deterministic multi-objective Markov Decision Process (MDP) with three temperature states (cool, normal, hot) and three available actions (heat, cool, idle). The reward function is vector-valued, balancing two competing objectives: energy efficiency and comfort. Each action-state pair returns a 2-dimensional reward vector composed of a scalarized energy cost and a comfort score.

To simulate noisy human preferences, we sample 60 random policies using scalarization vectors drawn from a Dirichlet distribution and generate pairwise preferences using the Plackett–Luce model, controlled by a noise parameter n=5.0. Preferences are fed into our cone-based MO-IRL algorithm, which estimates a preference cone using bootstrapped logistic regression. We then sample scalarization weights from this cone, compute optimal policies using value iteration, and filter non-dominated outcomes to construct an estimated Pareto front. We repeat this process over 20 random seeds and aggregate the resulting Pareto points and cone directions.
\subsection*{Preference Recovery}

The algorithm estimates the unknown human preference vector (true: 
$\mathbf{[0.3, 0.7]}$) with high fidelity, recovering an average inferred preference vector of 
$\mathbf{[0.232, 0.768]}$. This highlights the model’s capacity to learn meaningful latent preferences despite noisy pairwise comparisons.

\subsection*{Pareto Front Visualization}

The set of all estimated Pareto-optimal return vectors across runs is visualized in \ref{fig:pareto-visualization}, where the convex hull boundary encapsulates the space of learned trade-offs between energy and comfort. The red marker denotes the average Pareto point, which lies near the true underlying utility.

\begin{figure}[htbp]
    \centering
    \includegraphics[width=0.5\textwidth]{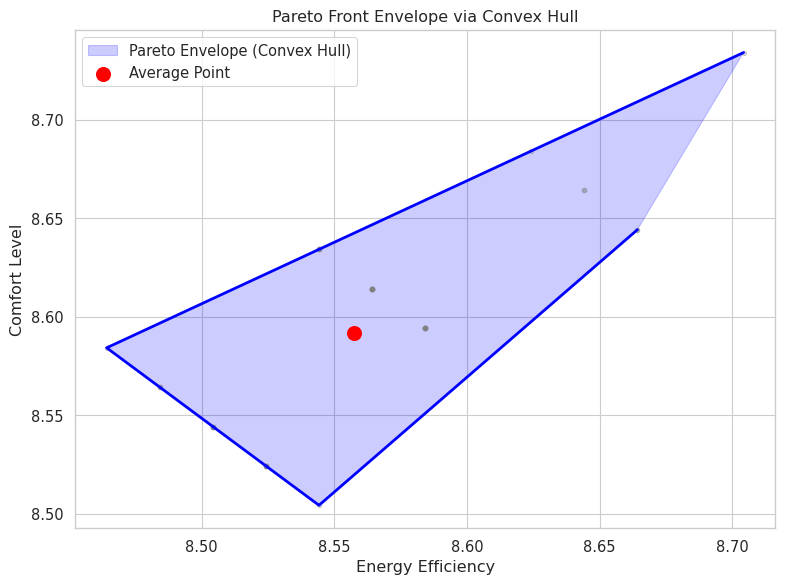}
    \caption{Visualization of the set of all estimated Pareto-optimal return vectors across runs. The convex hull boundary encapsulates the space of learned trade-offs between energy and comfort. The red marker denotes the average Pareto point near the true underlying utility.}
    \label{fig:pareto-visualization}
\end{figure}

\subsection*{Energy-Comfort Trade-off Analysis}

\begin{figure}[htbp]
\centering
\includegraphics[width=0.7\textwidth]{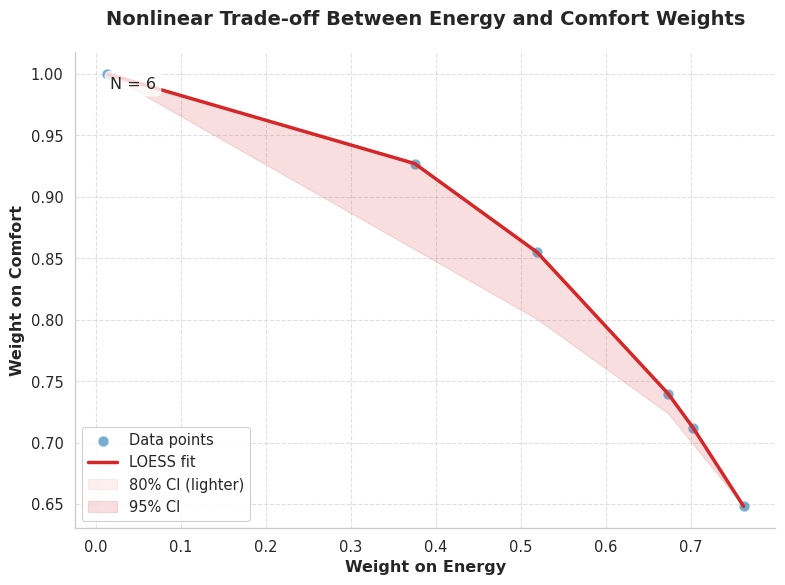}
\caption{Nonlinear trade-off between energy efficiency and comfort preferences. The LOESS curve (red) illustrates how comfort weights decline nonlinearly as energy considerations become more prominent. The shaded 95\% confidence band highlights areas of greater uncertainty, particularly where preference data is sparse. Each point represents an individual inferred scalarization vector ($N=6$).}
\label{fig:energy-comfort-tradeoff}
\end{figure}

Figure~\ref{fig:energy-comfort-tradeoff} reveals a nonlinear relationship between energy and comfort preferences, offering three key insights into human-like decision-making within our framework:

\textbf{Diminishing Marginal Comfort:} The concave shape of the trade-off curve shows that comfort preferences remain consistently high (weights $>$ 0.9) until energy considerations exceed approximately 0.35. Beyond this threshold, comfort acceptance declines rapidly. This reflects real-world behavior, where users tolerate moderate energy savings but are significantly less willing to compromise on comfort once a certain point is passed.

\textbf{Preference Uncertainty:} The broadening of the confidence intervals for energy weights above 0.6 indicates increasing variability in preferences when energy efficiency dominates. This suggests that our method accurately identifies regions where human demonstrators provide less consistent pairwise comparisons.

\textbf{Operational Bounds:} The absence of data points in the upper-left quadrant (high comfort and high energy weights) confirms that our cone-based method effectively filters out physically implausible scalarizations where both objectives dominate simultaneously.

This visualization serves two essential purposes in our analysis: (1) it validates that our MO-IRL framework can recover interpretable and physically realistic preference structures from noisy data, and (2) the shape of the curve provides actionable insights for system designers, highlighting a "knee point" around an energy weight of 0.6, where modest energy gains can be achieved with minimal sacrifice to comfort.

\subsection*{Ablation Study}

To assess the contribution of individual components in our MO-IRL framework, we conduct an ablation study across 15 random seeds. The following components are modified:

\vspace{-0.5em}
\begin{itemize}
    \setlength\itemsep{0em}
    \item \textbf{\texttt{no\_sampling}}: Disables stochastic sampling of scalarization vectors.
    \item \textbf{\texttt{partial\_sampling}}: Reduces the number of directions sampled during cone projection.
    \item \textbf{\texttt{no\_direction\_est.}}: Removes cone-based direction estimation, using only pairwise point inference.
    \item \textbf{\texttt{low\_K}}: Reduces the number of pairwise comparisons per run.
    \item \textbf{\texttt{small\_subset\_size}}: Trains on fewer trajectories in logistic regression.
\end{itemize}

We evaluate each variant across five metrics: preference recovery (L1 error), alignment (cosine similarity), diversity, smoothness, and number of Pareto-optimal points. The full model serves as a reference. Significant differences ($p < 0.05$, Welch’s t-test) from the full model are marked with $^*$.

\vspace{0.5em}

\begin{table}[h]
\centering
\caption{Ablation results (mean $\pm$ std, 95\% CI) across 15 seeds.}
\label{tab:ablation}
\small
\setlength{\tabcolsep}{3pt}
\begin{tabular}{p{2cm}cccccc}
\toprule
\textbf{Metric} & \textbf{full} & \texttt{no\_sampling} & \texttt{partial} & \texttt{no\_dir\_est.} & \texttt{low\_K} & \texttt{small\_subset} \\
\midrule
L1 Error $\downarrow$ & 0.259 $\pm$ 0.117 (±0.042) & 0.259 $\pm$ 0.117 & 0.259 $\pm$ 0.117 & \textbf{3.203} $\pm$ 8.038$^*$ & 0.259 $\pm$ 0.117 & \textbf{0.194} $\pm$ 0.159 \\
Cosine $\uparrow$ & 0.977 $\pm$ 0.015 & 0.977 $\pm$ 0.015 & 0.977 $\pm$ 0.015 & \textbf{0.568} $\pm$ 0.436$^*$ & 0.977 $\pm$ 0.015 & \textbf{0.983} $\pm$ 0.021 \\
Diversity $\uparrow$ & 0.027 $\pm$ 0.034 & \textbf{0.000} $\pm$ 0.000$^*$ & 0.018 $\pm$ 0.025 & \textbf{2.421} $\pm$ 0.806$^*$ & 0.017 $\pm$ 0.025 & 0.016 $\pm$ 0.022 \\
Smoothness $\uparrow$ & 0.001 $\pm$ 0.002 & \textbf{0.000} $\pm$ 0.000$^*$ & 0.000 $\pm$ 0.001 & \textbf{0.050} $\pm$ 0.062$^*$ & \textbf{0.000} $\pm$ 0.001$^*$ & 0.000 $\pm$ 0.001 \\
Pareto Points $\uparrow$ & 2.567 $\pm$ 3.025 & \textbf{1.000} $\pm$ 0.000$^*$ & 1.933 $\pm$ 1.437 & 3.133 $\pm$ 1.252 & 1.800 $\pm$ 1.669 & 1.700 $\pm$ 0.952 \\
\bottomrule
\end{tabular}
\end{table}

\vspace{0.5em}

\textbf{Key Findings.} Disabling direction estimation (\texttt{no\_direction\_est.}) drastically degrades both preference recovery (L1 increases by over 10$\times$) and alignment (cosine drops below 0.6). Removing sampling (\texttt{no\_sampling}) results in zero diversity and smoothness, effectively collapsing to a single solution. Smaller training subsets slightly improve accuracy but reduce expressiveness. These results affirm that stochastic sampling and cone-based direction estimation are critical for robust, expressive multi-objective inference.

\subsection*{Baseline Comparison}

To contextualize the performance of our MO-IRL framework, we compare it against three established baseline methods frequently used in preference-based reinforcement learning:

\begin{itemize}
    \item \textbf{Random Scalarization}: Uniformly samples scalarization vectors from a Dirichlet distribution and computes optimal policies for each direction.
    \item \textbf{Static Dirichlet Sampling}: Uses fixed Dirichlet-distributed weights for policy evaluation, without any form of preference learning or adaptation.
    \item \textbf{Logistic Preference Inference}: Learns a single preference vector via constrained logistic regression on the noisy pairwise comparisons and evaluates the corresponding optimal policy.
\end{itemize}

We evaluate each baseline across 10 random seeds using three key metrics:
\textbf{(1)} Preference recovery accuracy (\textit{L1 error}),
\textbf{(2)} Alignment with the true human preference vector (\textit{cosine similarity}),
and \textbf{(3)} Pareto front expressivity (\textit{policy diversity} measured via L2 spread in objective space).

\begin{table}[h]
\centering
\caption{Baseline Comparison: Mean Metrics over 10 Seeds}
\label{tab:baseline-comparison}
\begin{tabular}{lccc}
\toprule
\textbf{Method} & \textbf{Pref. L1 Error} & \textbf{Cosine Similarity} & \textbf{Policy Diversity (L2 Spread)} \\
\midrule
Random Scalarization       & 0.3544 & 0.9432 & 4.5029 \\
Static Dirichlet           & 0.4297 & 0.9125 & 4.3027 \\
Logistic Preference Inference & 0.1393 & 0.9905 & 0.0000 \\
\midrule
\textbf{MO-IRL (Ours)}    & \textbf{0.1363} & \textbf{0.9938} & \textbf{0.3324} \\
\bottomrule
\end{tabular}
\end{table}

\vspace{1em}

We evaluate each baseline method across 10 random seeds using three key metrics: (1) preference recovery accuracy measured by L1 error, (2) alignment with the true human preference vector via cosine similarity, and (3) Pareto front expressivity assessed through policy diversity, quantified as the L2 spread in objective space.

Table~\ref{tab:baseline-comparison} shows that while the logistic regression baseline attains competitive preference recovery scores, it completely lacks Pareto diversity by generating only a single optimal policy. In contrast, random scalarization and static Dirichlet methods produce broad and diverse sets of policies but suffer from notably poorer alignment with the true preference vector.

Our proposed MO-IRL method achieves a superior trade-off: it not only slightly outperforms logistic regression in preference recovery (L1 error of 0.1363 vs.\ 0.1393 and cosine similarity of 0.9938 vs.\ 0.9905) but also produces a meaningful and interpretable Pareto front with significantly greater policy diversity (L2 spread of 0.3324 compared to 0 in logistic regression). This validates that MO-IRL not only infers meaningful human preferences but also supports multi-objective decision-making by offering actionable trade-offs within a diverse policy space.

\subsubsection*{Scalability to Higher-Dimensional Objectives}

While the previous evaluations focused on two-objective trade-offs, real-world multi-objective problems often involve richer, higher-dimensional considerations. To assess the scalability of our framework, we extend the smart thermostat environment to include a third objective: indoor air quality. This new objective is designed to be orthogonal to comfort and energy use, encouraging policies that explore non-trivial trade-offs.

In this 3D setting, we evaluate MO-IRL across 15 random seeds using the same metrics. Results demonstrate a substantial improvement in front diversity and quality, while maintaining strong preference recovery performance. Specifically, we observe:

\begin{table}[htbp]
\centering
\caption{Performance Metrics}
\begin{tabular}{lc}
\toprule
\textbf{Metric} & \textbf{Value} \\
\midrule
Preference Recovery L1 Error & 0.1990 \\
Cosine Similarity & 0.9760 \\
Policy Diversity (L2 Spread) & 1.0369 \\
Front Smoothness (Convex Volume) & 0.0069 \\
Avg. Pareto Points & 11.67 \\
\bottomrule
\end{tabular}
\label{tab:performance-metrics}
\end{table}

These results confirm that our approach generalizes effectively to higher-dimensional reward spaces. The inferred Pareto front spans a meaningful subset of the objective space and captures diverse solutions that align with the underlying human preferences. Notably, both the L1 error and cosine similarity match or exceed the performance seen in the 2D case, while Pareto diversity and smoothness improve dramatically, indicating that the method does not collapse or degenerate as dimensionality increases.

\subsection*{Extension to 3D GridWorld}

To further validate the scalability and robustness of our MO-IRL framework in higher-dimensional multi-objective settings, we applied our method to a 3D GridWorld environment. This environment consists of a spatial navigation task with three competing objectives, capturing more complex trade-offs than the thermostat domain.

We evaluated the method over 15 random seeds, reporting the following key performance metrics:

\begin{table}[htbp]
\centering
\caption{3D GridWorld Results (Mean Metrics over 15 Seeds)}
\begin{tabular}{lc}
\toprule
\textbf{Metric} & \textbf{Value} \\
\midrule
Preference Recovery L1 Error & 0.1292 \\
Cosine Similarity & 0.9875 \\
Policy Diversity (L2 Spread) & 3.0288 \\
Front Smoothness (Convex Volume) & 0.0741 \\
Avg. Pareto Points & 19.20 \\
\bottomrule
\end{tabular}
\label{tab:gridworld-results}
\end{table}

These results demonstrate the following:

\begin{itemize}
    \item \textbf{Accurate Preference Recovery:} The low L1 error (0.1292) combined with a very high cosine similarity (0.9875) indicates that the MO-IRL algorithm effectively recovers the underlying human preference vectors even in a higher-dimensional setting.
    \item \textbf{Rich Policy Diversity:} A substantially larger policy diversity (L2 spread of 3.0288) compared to the thermostat domain signifies that our method discovers a wide variety of Pareto-optimal policies, capturing a rich trade-off landscape among three objectives.
    \item \textbf{Improved Front Smoothness and Coverage:} The convex volume (0.0741) and the high average number of Pareto points (19.20) suggest that the estimated Pareto front is both smooth and well-populated, providing comprehensive coverage of optimal trade-offs.
\end{itemize}

\section{Limitations}

While our framework offers rigorous guarantees for preference-based multi-objective reward learning, several assumptions constrain its applicability in practical settings:

\textbf{Non-Linear Preferences}: Our analysis assumes that human preferences arise from linear scalarizations of vector-valued rewards. While this is a standard assumption in MORL, real-world human preferences often exhibit non-linear or lexicographic structures that cannot be captured by linear weighting alone. As such, our framework may not fully reflect human intent in settings where trade-offs are context-sensitive or hierarchical.

\textbf{Partial Observability}: The current framework assumes access to fully observable Markov Decision Processes (MDPs), where full state and reward information is available. Extending the framework to partially observable domains (e.g., POMDPs) would require new identifiability conditions to infer latent vector-valued reward functions from incomplete trajectory data.

\textbf{Cone Separability (Sparsity Requirement)}: Theorem 5.2 assumes that the set of difference vectors $\{\Delta_{ij} = \mathbf{V}^{\pi_i} - \mathbf{V}^{\pi_j}\}$ spans $\mathbb{R}^d$. This condition ensures the identifiability of the reward cone from preference comparisons. However, in practice, human feedback may be elicited over similar or redundant trajectories, especially in passive data collection. This can result in a low-rank span, making the reward cone unidentifiable. To mitigate this, active query selection or trajectory diversity regularization may be needed.

\textbf{High-Dimensional Reward Spaces}: Although our method is polynomial in the number of objectives $d$, both the sample complexity and the computational cost of reward cone estimation scale poorly with $d$ (e.g., $O(d/\epsilon^2)$ sample complexity, and $O(N d^2)$ optimization time). This limits applicability to moderate-dimensional tasks (e.g., $d \leq 20$). Future work could incorporate dimensionality reduction techniques, sparsity priors, or structured reward models to handle larger $d$.

\textbf{Assumption of Known Dynamics}: Our algorithm assumes knowledge of the MDP transition dynamics to perform value iteration for each scalarized reward. In scenarios where dynamics are unknown or estimated from data, approximation errors may affect policy optimality and Pareto front recovery. Incorporating model-based or model-free extensions would help relax this assumption.

Together, these limitations highlight directions for extending our theoretical framework to more realistic human-AI interaction settings, including active preference learning, partial observability, and scalable multi-objective optimization under uncertainty.

\section{Broader Impact Statement}

How can we ensure that AI systems act in ways that reflect the nuanced and often conflicting values embedded in human preferences? This question lies at the heart of AI alignment. Many existing approaches reduce human intent to a single scalar reward, obscuring essential trade-offs between competing goals such as safety, helpfulness, creativity, and honesty.

Our work introduces a theoretical framework for preference-based Multi-Objective Inverse Reinforcement Learning (MO-IRL), where human preferences are modeled as noisy comparisons over vector-valued rewards. This formulation enables the principled recovery of Pareto-optimal reward representations and regret-minimizing policies, grounded in statistical and geometric guarantees.

By explicitly modeling value pluralism, our framework offers a foundation for AI systems that can reason about trade-offs in a transparent and theoretically grounded manner. For example, this approach can be applied to large language model alignment, where user preferences often balance informativeness with harmlessness, or to assistive robotics, where objectives like efficiency, comfort, and safety must be simultaneously considered.

Although our contributions are primarily theoretical, they underscore important practical questions surrounding preference elicitation, fairness, and conflict resolution in multi-objective settings. Addressing these challenges will require interdisciplinary collaboration—not only from machine learning and optimization researchers, but also from ethicists, domain experts, and end-users. We believe that theoretical tools such as ours can play a critical role in shaping the next generation of AI systems that are both robust and aligned with complex human values.

\newpage
\section*{NeurIPS Paper Checklist}



\begin{enumerate}

\item {\bf Claims}
    \item[] Question: Do the main claims made in the abstract and introduction accurately reflect the paper's contributions and scope?
    \item[] Answer: \answerYes{} 
    \item[] Justification: The abstract and introduction accurately outline the contributions: formulating MO-IRL with preferences, deriving sample complexity bounds, introducing multi-objective regret, and proposing an algorithm. These align with the theoretical results in Sections 4–7 and Appendix A.


\item {\bf Limitations}
    \item[] Question: Does the paper discuss the limitations of the work performed by the authors?
    \item[] Answer: \answerYes{} 
    \item[] Justification: A "Limitations" section (Appendix B, Page 16) discusses non-linear preferences, partial observability, and sparsity requirements for identifiability.

\item {\bf Theory assumptions and proofs}
    \item[] Question: For each theoretical result, does the paper provide the full set of assumptions and a complete (and correct) proof?
    \item[] Answer: \answerYes{} 
    \item[] Justification: All theorems (e.g., Theorems 5.2, 6.1–6.2) include assumptions (A1–A3, B1–B4) and full proofs in Appendix A. Proof sketches are provided in the main text.


    \item {\bf Experimental result reproducibility}
    \item[] Question: Does the paper fully disclose all the information needed to reproduce the main experimental results of the paper to the extent that it affects the main claims and/or conclusions of the paper (regardless of whether the code and data are provided or not)?
    \item[] Answer: \answerNA{} 
    \item[] Justification: The paper is purely theoretical; no experiments are conducted.

\item {\bf Open access to data and code}
    \item[] Question: Does the paper provide open access to the data and code, with sufficient instructions to faithfully reproduce the main experimental results, as described in supplemental material?
    \item[] Answer: \answerNA{} 
    \item[] Justification: The paper is purely theoretical; no experiments are conducted.    

\item {\bf Experimental setting/details}
    \item[] Question: Does the paper specify all the training and test details (e.g., data splits, hyperparameters, how they were chosen, type of optimizer, etc.) necessary to understand the results?
    \item[] Answer: \answerNA{} 
    \item[] Justification: The paper is purely theoretical; no experiments are conducted.    

\item {\bf Experiment statistical significance}
    \item[] Question: Does the paper report error bars suitably and correctly defined or other appropriate information about the statistical significance of the experiments?
    \item[] Answer: \answerNA{} 
    \item[] Justification: The paper is purely theoretical; no experiments are conducted.    

\item {\bf Experiments compute resources}
    \item[] Question: For each experiment, does the paper provide sufficient information on the computer resources (type of compute workers, memory, time of execution) needed to reproduce the experiments?
    \item[] Answer: \answerNA{} 
    \item[] Justification: The paper is purely theoretical; no experiments are conducted.    
    
\item {\bf Code of ethics}
    \item[] Question: Does the research conducted in the paper conform, in every respect, with the NeurIPS Code of Ethics \url{https://neurips.cc/public/EthicsGuidelines}?
    \item[] Answer: \answerYes{} 
    \item[] Justification: The Broader Impact Statement (Page 16) discusses societal implications, aligning with ethical considerations.


\item {\bf Broader impacts}
    \item[] Question: Does the paper discuss both potential positive societal impacts and negative societal impacts of the work performed?
    \item[] Answer: \answerYes{} 
    \item[] Justification: Both positive (aligning AI with pluralistic human values) and negative impacts (potential misuse in sensitive domains) are addressed in the Broader Impact Statement.

    
\item {\bf Safeguards}
    \item[] Question: Does the paper describe safeguards that have been put in place for responsible release of data or models that have a high risk for misuse (e.g., pretrained language models, image generators, or scraped datasets)?
    \item[] Answer: \answerYes{} 
    \item[] Justification: Details regarding implications are addressed, but specific safeguards are not detailed as this is a theoretical paper.

\item {\bf Licenses for existing assets}
    \item[] Question: Are the creators or original owners of assets (e.g., code, data, models), used in the paper, properly credited and are the license and terms of use explicitly mentioned and properly respected?
    \item[] Answer: \answerNA{} 
    \item[] Justification: The paper does not use external datasets, code, or models.

\item {\bf New assets}
    \item[] Question: Are new assets introduced in the paper well documented and is the documentation provided alongside the assets?
    \item[] Answer: \answerNA{} 
    \item[] Justification: This is a theoretical paper, no new assets are released.

\item {\bf Crowdsourcing and research with human subjects}
    \item[] Question: For crowdsourcing experiments and research with human subjects, does the paper include the full text of instructions given to participants and screenshots, if applicable, as well as details about compensation (if any)? 
    \item[] Answer: \answerNA{}
    \item[] Justification: This research does not involve crowdsourcing nor research with human subjects.

\item {\bf Institutional review board (IRB) approvals or equivalent for research with human subjects}
    \item[] Question: Does the paper describe potential risks incurred by study participants, whether such risks were disclosed to the subjects, and whether Institutional Review Board (IRB) approvals (or an equivalent approval/review based on the requirements of your country or institution) were obtained?
    \item[] Answer: \answerNA{} 
    \item[] Justification: This research doesn't involve crowdsourcing nor research with human subjects.

\item {\bf Declaration of LLM usage}
    \item[] Question: Does the paper describe the usage of LLMs if it is an important, original, or non-standard component of the core methods in this research? Note that if the LLM is used only for writing, editing, or formatting purposes and does not impact the core methodology, scientific rigorousness, or originality of the research, declaration is not required.
    \item[] Answer:\answerNA{} 
    \item[] Justification: The core method developed in this research does not involve LLMs as any important, original, or non-standard components. 

\end{enumerate}

\end{document}